\def\year{2022}\relax
\documentclass[letterpaper]{article} 
\usepackage{aaai22}  
\usepackage{times}  
\usepackage{helvet}  
\usepackage{courier}  
\usepackage[hyphens]{url}  
\usepackage{graphicx} 
\usepackage{natbib}  
\usepackage{caption} 
\DeclareCaptionStyle{ruled}{labelfont=normalfont,labelsep=colon,strut=off} 
\usepackage[ruled,noend]{algorithm2e}

\usepackage[noend]{algpseudocode}
\usepackage{xcolor}
\usepackage{amssymb}
\usepackage{amsfonts,amsmath,bm}
\usepackage{amsthm}
\usepackage{subfig}
\usepackage{enumitem}

\newtheorem{example}{Example}
\newtheorem{definition}{Definition}
\newtheorem{property}{Property}
\newtheorem{lemma}{Lemma}
\newtheorem{theorem}{Theorem}

\newcommand{\Open}{\ensuremath{Open}\xspace}

\newcommand{\Solutions}{\ensuremath{S}\xspace}
\newcommand{\comax}{\ensuremath{\mathbf{comax}}}
\newcommand{\ND}{\ensuremath{\mathbf{ND}}}

\newcommand{\Piset}{\Pi}

\usepackage{newfloat}
\usepackage{listings}
\title{Cost Splitting for Multi-Objective Conflict-Based Search}

\title{Cost Splitting for Multi-Objective Conflict-Based Search}
\author {
    Cheng Ge,\textsuperscript{\rm 1 *}
    Han Zhang, \textsuperscript{\rm 2} \footnote{These authors contributed equally to this work and should be considered co-first authors.}
    Jiaoyang Li, \textsuperscript{\rm 3}
    Sven Koenig, \textsuperscript{\rm 2}
}
\affiliations {
    \textsuperscript{\rm 1} Institute for Interdisciplinary Information Science, Tsinghua University\\
    \textsuperscript{\rm 2} Computer Science Department, University of Southern California\\
    \textsuperscript{\rm 3} Robotics Institute at Carnegie Mellon University\\    
    ge\_mike@outlook.com, \{zhan645, skoenig\}@usc.edu,  jiaoyangli@cmu.edu
}

\usepackage{bibentry}

\begin{document}

\maketitle

\begin{abstract}
The Multi-Objective Multi-Agent Path Finding (MO-MAPF) problem is the problem of finding the Pareto-optimal frontier of collision-free paths for a team of agents while minimizing multiple cost metrics. 
Examples of such cost metrics include arrival times, travel distances, and energy consumption.
In this paper, we focus on the Multi-Objective Conflict-Based Search (MO-CBS) algorithm, a state-of-the-art MO-MAPF algorithm.
We show that the standard splitting strategy used by MO-CBS can lead to duplicate search nodes and hence can duplicate the search effort that MO-CBS needs to make.
To address this issue, we propose two new splitting strategies for MO-CBS, namely cost splitting and disjoint cost splitting.
Our theoretical results show that, when combined with either of these two new splitting strategies, MO-CBS maintains its completeness and optimality guarantees. Our experimental results show that disjoint cost splitting, our best splitting strategy, speeds up MO-CBS by up to two orders of magnitude and substantially improves its success rates in various settings.

\end{abstract}

\section{Introduction}

The Multi-Agent Path Finding (MAPF) problem is the problem of finding a set of collision-free paths for a team of agents. 
It is related to many real-world applications~\cite{wurman2008coordinating,morris2016planning}. 
Solving it optimally is known to be NP-hard for various objective functions~\cite{yu2013structure,ma2016multi}. 
In this paper, we focus on a variant of the MAPF problem called the Multi-Objective MAPF (MO-MAPF) problem~\cite{ren2021multi}. 
In MO-MAPF, a solution is a set of collision-free paths for all agents, and we consider multiple cost metrics for each solution.
Such cost metrics can be arrival times, travel distances, and other domain-specific metrics. 
The objective of MO-MAPF is to find the Pareto-optimal solution set, that is, all solutions that are not dominated by any other solutions, where a solution $\Pi$ dominates another solution $\Pi'$ iff the cost of $\Pi$ is no larger than the cost of $\Pi'$ for every cost metric and at least one of them is smaller.

There are two existing algorithms for solving the MO-MAPF problem. MO-M*~\cite{ren2021subdimensional} generalizes M*~\cite{wagner2015subdimensional}, an algorithm for solving MAPF optimally, to MO-MAPF. However, experimental results~\cite{ren2021multi} show that MO-M* runs slower than Multi-Objective Conflict-Based Search (MO-CBS), the other MO-MAPF algorithm, in various domains.

MO-CBS~\cite{ren2021multi} generalizes Conflict-Based Search (CBS) \cite{sharon2015conflict}, which is also an algorithm for solving MAPF optimally, to MO-MAPF.
Here, conflicts refer to collisions between agents.
In MAPF, CBS first finds a minimum-cost path for each agent without considering conflicts. 
To resolve a conflict, CBS branches into two subproblems, each with a new constraint imposed on either one of the conflicting agents to prevent the conflict from happening again. It replans for the constrained agent in each subproblem by finding a minimum-cost path that satisfies all constraints in the subproblem.
CBS repeats the conflict resolution process until it finds a solution.
In MO-MAPF, MO-CBS is similar to CBS and resolves conflicts by adding constraints to agents.
One difference between them is that, when replanning for the constrained agent, MO-CBS finds all Pareto-optimal paths that satisfy the constraints and creates a new subproblem for each of them. 
The resultant subproblems have the same set of constraints and differ from each other only in terms of the path of the constrained agent.
As we will show in the paper, these similar subproblems can lead to substantial duplicate search effort. 


In this paper, we dramatically speed up MO-CBS by improving its splitting strategy. 
We propose two new splitting strategies for MO-CBS, namely cost splitting and disjoint cost splitting. 
Both cost splitting and disjoint cost splitting impose additional constraints on the costs of the paths for the agents during splitting.
Therefore, each subproblem is more constrained, and, when resolving conflicts in its following subproblems, MO-CBS generates fewer subproblems and hence has less duplicate search effort. 
Our theoretical results show that, when combined with either of these two new splitting strategies, MO-CBS maintains its completeness and optimality.
Our experimental results show that these two proposed techniques substantially improve the success rates and runtimes of MO-CBS in various domains. In many instances, the runtime speedups are more than $25\times$ and the maximum runtime speedup is more than $125\times$.

\section{Terminology and Problem Definition}

In this paper, we use boldface fonts to denote vectors, sets of vectors, or vector functions. We use $v_i$ to denote the $i$-th component of vector or vector function $\mathbf{v}$. Given two vectors  $\mathbf{v}$ and $\mathbf{v'}$ of the same length $N$, their addition is defined as $\mathbf{v} + \mathbf{v'} = [v_1 + v'_1, v_2 + v'_2\ldots v_N + v'_N]$, and their component-wise maximum is defined as $\comax(\mathbf{v}, \mathbf{v'}) =  [\max(v_1, v_1')\ldots \max(v_N, v_N')]$.
We say that $\mathbf{v}$ \emph{weakly dominates} $\mathbf{v'}$, that is, $\mathbf{v} \preceq \mathbf{v'}$, iff $v_i \leq v'_i$ for all $i=1,2\ldots N$. 
We say that $\mathbf{v}$ \emph{dominates} $\mathbf{v'}$, that is, $\mathbf{v}\prec \mathbf{v'}$, iff $\mathbf{v} \preceq \mathbf{v'}$ and there exists an $i\in \{1,2\ldots N\}$ with $v_i < v'_i$.
Given a set of vectors $\mathbf{S}$, we use $\ND(\mathbf{S})= \{\mathbf{v} \in \mathbf{S} \mid \forall \mathbf{u} \in \mathbf{S} , \mathbf{u} \not\prec \mathbf{v}\}$ to denote the set of all nondominated vectors in $\mathbf{S}$. 
\begin{property}
Given three vectors $\mathbf{v}_1$, $\mathbf{v}_2$, and $\mathbf{u}$, we have $\comax(\mathbf{v}_1, \mathbf{v}_2) \preceq \mathbf{u}$ iff $\mathbf{v}_1 \preceq \mathbf{u}$ and $\mathbf{v}_2 \preceq \mathbf{u}$
\label{prop0}
\end{property}

The MO-MAPF problem is defined by a weighted directed graph $G = \langle V,E \rangle$ and a set of $m$ agents $\{a_1 \ldots a_m\}$, each $a_i$ associated with a start vertex $s_i \in V$ and a goal vertex $g_i \in V$. $V$ is the set of vertices that the agents can stay at, and $E$ is the set of directed edges that the agents can move along, each associated with an edge cost, which is a non-negative vector of length $N$. An edge pointing from $u \in V$ to $v \in V$ is denoted as $\langle u, v\rangle \in E$. 
Additionally, we allow a self-pointing edge at each vertex, which corresponds to the case when the agent waits at the current vertex. 

In each timestep, an agent either moves along an edge (moving to a neighboring vertex or waiting at its current vertex) or terminates at its goal vertex (and does not move any longer).
The cost of an action is the cost of the corresponding edge if the action is to move along an edge and a zero vector of length $N$ if it is a terminate action.
A \emph{path} $\pi$ of an agent is a sequence of actions that leads it from its start vertex to its goal vertex and ends with a terminate action. 
Its \emph{cost} $\mathbf{c}(\pi)$ is the sum of the costs of its actions.
There are two types of conflicts.
A \emph{vertex conflict} happens when two agents stay at the same vertex simultaneously, and an \emph{edge conflict} happens when two agents swap their vertices simultaneously.
A \emph{solution} $\Pi$ is a set of conflict-free paths of all agents. Its \emph{cost} $\mathbf{c}(\Pi)$ is the sum of the costs of its paths. 

Given a MO-MAPF problem instance (respectively a MO-MAPF problem instance and an agent), a solution (respectively path) is \emph{Pareto-optimal} iff its cost is not dominated by the cost of any other solution (respectively path); 
the \emph{Pareto-optimal frontier} is the set of all Pareto-optimal solutions (respectively paths); and a \emph{cost-unique Pareto-optimal frontier} is a maximal subset of the Pareto-optimal frontier that does not contain any two solutions (respectively paths) of the same cost.

The goal of the MO-MAPF problem is to find a cost-unique Pareto-optimal frontier (of solutions).
We can easily fit different objectives into this problem by assigning different edge costs. 
For example, we can measure \emph{flowtime}, which is commonly used as a MAPF optimization objective, as the $i$-th objective of the problem instance by setting the $i$-th digit of the cost of every edge to one.

\section{Algorithmic Background}

We review two existing algorithms on which we build our techniques, namely CBS for the MAPF problem and MO-CBS for the MO-MAPF problem. 

\subsection{CBS}

CBS is a complete and optimal two-level MAPF algorithm.
On the high level, CBS performs a best-first search on a \emph{Constraint Tree} (CT).
Each CT node $n$ contains a set of constraints $n.constraints$ and a set of paths $n.paths$, one for each agent, that satisfy the constraints. A \textit{vertex constraint } $\langle a_i, v, t\rangle$ prohibits agent $a_i$ from using vertex $v$ at timestep $t$, and an \textit{edge constraint } $\langle a_i, u, v, t\rangle$ prohibits agent $a_i$ from using edge $\langle u, v\rangle$ between timesteps $t$ and $t+1$.
The cost of a CT node is the sum of the costs of its paths.

CBS starts with a root CT node, which has an empty set of constraints and a minimum-cost path for each agent that ignores conflicts.
When expanding a CT node, CBS returns its paths as a solution if they are conflict-free.
Otherwise, CBS picks a conflict to resolve: It splits the CT node into two child CT nodes and adds a constraint to each child CT node to prohibit either one or the other of the two conflicting agents from using the conflicting vertex or edge at the conflicting timestep. 
CBS then calls its low level to replan the path of the newly constrained agent in each child CT node.
On the low level, CBS finds a minimum-cost path for the given agent that satisfies the constraints of the given CT node but ignores the conflicts with other agents.

\begin{algorithm}[t!]
    \footnotesize
   
\DontPrintSemicolon
\SetKwProg{Fn}{Function}{:}{end}
    \label{high_level}

            \textsc{Initialization()}\; \label{line:init}
            \Solutions $\gets \emptyset$\tcp*{Stores the solutions} \label{line:maintain_solution}
            \While{\Open is not empty}{
                $n \gets \Open.pop()$ \; \label{line:pop_node}
                \lIf{$\exists \ \Pi \in \Solutions, \mathbf{c}(\Pi) \preceq \mathbf{c}(n.paths)$
                \label{line:prune_pop}}{%
                    \textbf{continue}%
                }
                \If{$n$ is conflict-free}{
                     add $n.paths$ to \Solutions\; \label{line:add_solution}
                    \textbf{continue}
                    }
            
                $\mathit{conf} \gets $ a conflict in $n$\; \label{line:start_expand} \label{line:resolve_conflict}
                $children \gets \textsc{Split}(n, \mathit{conf} )$\; \label{line:get_child_node}
                 \ForEach{ $n'\in children$}{
                 \lIf{$\exists \ \Pi \in \Solutions, \mathbf{c}(\Pi) \preceq \mathbf{c}(n.paths)$
                 \label{line:prune_add}}{%
                            \textbf{continue}%
                        }
                 add $n'$ to \Open\; \label{line:end_expand}
                 }

            }
             \Return \Solutions\;
             
    \Fn{\textsc{Split}$(n, \mathit{conf})$}{
        $children \gets \emptyset$\;
         \ForEach{ $a_i$ involved in $\mathit{conf}$ \label{line:high_level_split}}{
            $\mathit{cons} \gets$ the constraint imposed on $a_i$\;
            $C'_i \gets n.constraints \cup \{\mathit{cons}\}$\;
            $\Piset'_i \gets \textsc{LowLevelSearch}(a_i, C'_i)$\; \label{line:low_level_search}
            \ForEach{ $\pi_i\in \Piset'_i$\label{line:new_child_node}}{
                $n' \gets n$\;
                $n'.constraints\gets C'_i$\;
             $n'.paths[i]\gets \pi_i$\;
                add $n'$ to $children$\; 
            }
        }
    \Return $children$\;
    }
\caption{MO-CBS with standard splitting}
\label{alg:mo-cbs}
\end{algorithm}

\subsection{MO-CBS}

MO-CBS extends CBS to the MO-MAPF problem. 
Algorithm~\ref{alg:mo-cbs} shows its high level.
MO-CBS first calls function \textsc{Initialization} to generate (potentially multiple) root CT nodes. Specifically, MO-CBS calls its low level to find a cost-unique Pareto-optimal frontier $\Piset_i$ for each agent $a_i$.
For each combination of paths in $\Piset_1\times\Piset_2\ldots\Piset_m$, MO-CBS generates a root CT node that has an empty set of constraints and the corresponding path for each agent and inserts the root CT node into the $Open$ list. In each iteration, a CT node with the lexicographically smallest cost is extracted from $Open$ for expansion.

Similar to CBS, when expanding a CT node $n$ whose paths are not conflict-free, MO-CBS picks a conflict to resolve.
The splitting strategy of MO-CBS is two-level.
MO-CBS first splits a CT node into two subproblems, each with a new constraint on either one of the conflicting agents. Then, for each constraint, MO-CBS calls its low level to replan a cost-unique Pareto-optimal frontier for the constrained agent. For each path in it, MO-CBS generates a new child CT node. 
When expanding a CT node whose paths are conflict-free, MO-CBS finds a new solution and adds it to the solution set.
MO-CBS prunes a CT node if its cost is weakly dominated by the cost of any solution in the solution set.
MO-CBS terminates when $Open$ is empty.

\begin{figure*}[t]
    \centering
      \subfloat[]{
    \includegraphics[width=0.3\textwidth]{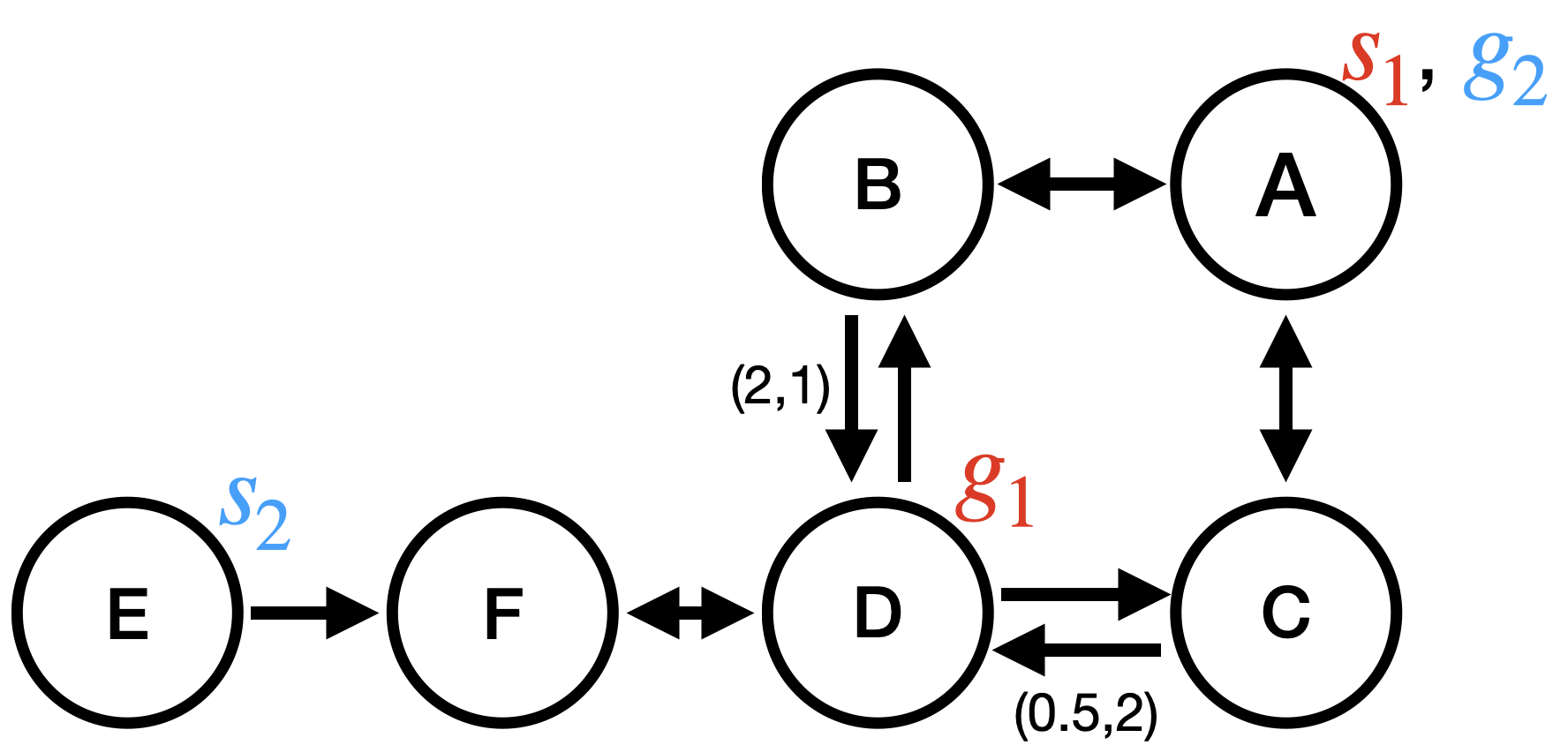}  
      }
    \subfloat[]{
    \includegraphics[width=0.6\textwidth]{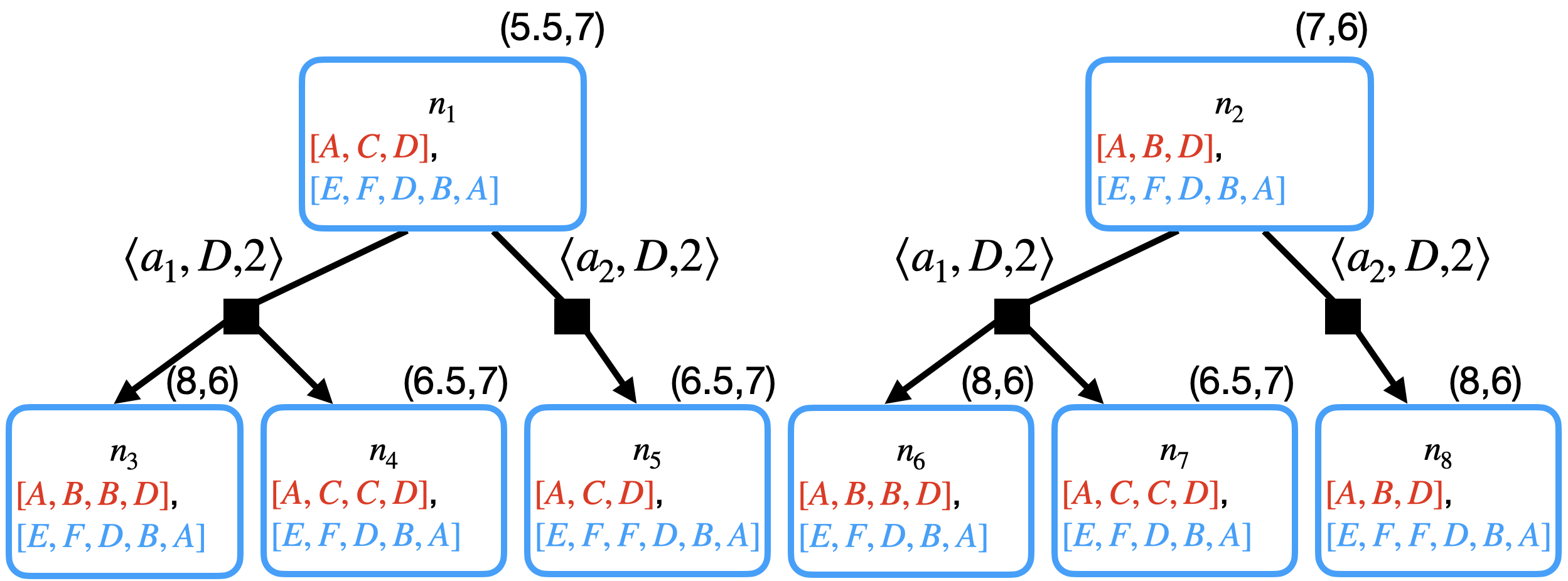}  
    \label{fig:example_ct}
      }
    \caption{An example MO-MAPF instance with two agents. (a) shows the graph for this problem instance. 
    The self-pointing edges are not shown, and all of them have costs $(1,0)$.
    Except $\mathbf{c}(\langle B,D \rangle) = (2,1) $ and  $\mathbf{c}(\langle C,D \rangle) = (0.5,2)$, all shown edges have costs $(1,1)$.
    The start and goal vertices of agent $a_1$ are A and D, respectively.
    The start and goal vertices of agent $a_2$ are E and A, respectively.
    (b) shows the CTs of MO-CBS for this MO-MAPF instance. Blue boxes represent CT nodes.
    The label inside each CT node denotes the name of the CT node and the paths for agents $a_1$ (in red) and $a_2$ (in blue).
    For ease of presentation, we show each path as the sequence of vertices that are traversed.
    The label next to each CT node is its cost.}
    \label{fig:example} 
\end{figure*}

\begin{example}
Figure~\ref{fig:example} shows an MO-MAPF instance with two agents. 
In the beginning, a cost-unique Pareto-optimal frontier for agent $a_1$ contains paths $[A,C,D]$ and $[A,B,D]$ with costs $(1.5,3)$ and $(3,2)$, respectively.
A cost unique Pareto-optimal frontier for agent $a_2$ contains only one path $[E,F,D, B, A]$ with cost $(4,4)$.
Therefore, MO-CBS generate two root CT nodes $n_1$ and $n_2$. 
We assume that MO-CBS breaks ties in favor of CT nodes that are generated earlier when extracting a CT node to expand.
\begin{enumerate}

\item  MO-CBS first expands CT node $n_1$ and picks the vertex conflict between agents $a_1$ and $a_2$ at vertex $D$ at timestep 2 to resolve. 
MO-CBS splits $n_1$ with vertex constraints $\langle a_1, D, 2\rangle$ and $\langle a_2, D, 2\rangle$. For agent $a_1$, MO-CBS finds a cost-unique Pareto-optimal frontier that consists of two paths $[A,B,B,D]$ and $[A,C,C,D]$ and thus generates two child CT nodes $n_3$ and $n_4$, one for each path.
For agent $a_2$, MO-CBS finds a cost-unique Pareto-optimal frontier consists of only one path $[E,F, F,D,B, A]$ and generates child CT node $n_5$.
\item MO-CBS expands CT node $n_4$ and finds a Pareto-optimal solution with cost $(6.5,7)$. 
\item MO-CBS extracts and then prunes CT node $n_5$ because its cost is weakly dominated by the cost of the solution found in CT node $n_4$.
\item  MO-CBS expands CT node $n_2$ and, similar to the expansion of $n_1$, generates three CT nodes $n_6$, $n_7$, and $n_8$.
\item MO-CBS extracts and then prunes CT node $n_7$ because its cost is weakly dominated by the cost of the solution found in CT node $n_4$.
\item MO-CBS expands CT node $n_3$ and finds a Pareto-optimal solution with cost $(8,6)$. 
\item MO-CBS extracts and then prunes CT nodes $n_6$ and $n_8$ because their costs are both weakly dominated by the cost of the solution found in CT node $n_3$.
\end{enumerate}
MO-CBS terminates and finds a cost-unique Pareto-optimal frontier with two solutions.
\label{example:initial}
\end{example}

\section{Cost Splitting}


The following example shows that MO-CBS sometimes generates identical CT nodes.

\begin{example}
Consider CT nodes $n_3$ and $n_6$ in Figure~\ref{fig:example_ct}. Both CT nodes contain the same sets of constraints and paths, which makes these two CT nodes indistinguishable from each other.
CT nodes $n_4$ and $n_7$ are indistinguishable from each other, too.
In our example, CT nodes $n_3$ and $n_6$ are both conflict-free. However, for an MO-MAPF instance with more agents, such duplicate CT nodes can contain conflicts between other agents, and, to resolve these conflicts, MO-CBS repeats the search effort in the search trees that are rooted in these duplicate CT nodes.
\end{example}

We now describe cost splitting, which generates fewer CT nodes than the standard splitting and still retains the completeness and optimality guarantees of MO-CBS.
For each CT node $n$, MO-CBS with cost splitting maintains a \emph{cost lower bound} $n.\mathbf{lb}=[n.\mathbf{lb}_1\ldots n.\mathbf{lb}_m]$ that consists of $m$ vectors of length $N$, one for each agent. It also modifies the \textsc{Initialization} and \textsc{Split} functions used by Algorithm~\ref{alg:mo-cbs}. 
\begin{definition}
A path $\pi_i$ of agent $a_i$ is compatible with a CT node $n$ iff:
\begin{enumerate}
    \item $n.\mathbf{lb}_i \preceq \mathbf{c}(\pi_i)$ and
    \item $\pi_i$ satisfies all constraints of CT node $n$.
\end{enumerate}
A solution is compatible with a CT node iff all its paths are compatible with the CT node.
\label{def:compatible_old}
\end{definition}

In \textsc{Initialization}, MO-CBS with cost splitting initializes the cost lower bound of each agent in each root CT node to the cost of its path.
For example, the cost lower bound of CT nodes $n_1$ and $n_2$ in Figure~\ref{fig:example_ct} are initialized to $[(1.5, 3), (4,4)]$ and $[(3, 2), (4,4)]$, respectively. 

\begin{property}
Any solution is compatible with at least one root CT node.
\label{prop1_5}
\end{property}
\begin{proof}
Any combination of the cost-unique Pareto-optimal frontiers of all agents corresponds to a root CT node, so any solution is weakly dominated by the paths and thus the cost lower bounds of at least one root CT node.
Since root CT nodes do not have any constraints, the statement holds. 
\end{proof}
\begin{algorithm}[t]
    \footnotesize
    
\DontPrintSemicolon
\SetKwProg{Fn}{Function}{:}{end}
   
    \Fn{\textsc{Split}$(n, \mathit{conf})$}{
        $children \gets \emptyset$\;
         \ForEach{ $a_i$ involved in $\mathit{conf}$}{
            $\mathit{cons} \gets$ the constraint imposed on $a_i$\;
            $C'_i \gets n.constraints \cup \{\mathit{cons}\}$ \label{line:constraint-union}\;
            $\Piset'_i \gets \textsc{LowLevelSearch}(a_i, C'_i)$ \label{line:cost_split_lowlevel} \;
            $\mathbf{LB}_i \gets \ND(\{\comax(n.\mathbf{lb}_i, \mathbf{c}(\pi)) \mid \mathbf{\pi} \in \Piset_i'\})$\;\label{line:LBi}
            \ForEach{ $\mathbf{lb}\in \mathbf{LB}_i$}{
                $n' \gets n$\;
                $n'.constraints\gets C'_i$\;
                $n'.\mathbf{lb}_i \gets \mathbf{lb}$\; \label{line:set_lower_bound}
             $n'.paths[i]\gets$ a path $ \pi\in \Piset'_i$ such that $ \comax(n.\mathbf{lb}_i, \mathbf{c}(\pi)) = \mathbf{lb}$\;\label{line:set_path}
                add $n'$ to $children$\;
            }
        }
    \Return $children$\;
    }
\caption{Cost splitting}
\label{alg:cost-split}
\end{algorithm}

Algorithm~\ref{alg:cost-split} describes the $\textsc{Split}$ function of cost splitting.
For each agent $a_i$ involved in the conflict and the corresponding constraint $cons_i$, cost splitting calls its low level to replan a cost-unique Pareto-optimal frontier $\Piset_i'$ that satisfies constraints $n.constraints \cup \{ cons_i\}$.
It then computes a set of vectors
$$\mathbf{LB}_i = \ND(\{\comax(n.\mathbf{lb}_i, \mathbf{c}(\pi))| \mathbf{\pi} \in \Piset_i'\}).$$
For each cost vector $\mathbf{lb}\in \mathbf{LB}_i$, cost splitting creates a child CT node whose cost lower bound for agent $a_i$ is set to $\mathbf{lb}$, and whose path of agent $a_i$ is set to the corresponding path in $\Piset_i'$.
We define the split strategy in this way because (1) it guarantees that MO-CBS does not exclude any compatible solution during CT node splitting, as indicated by Property~\ref{prop2}, and (2) it eliminates some duplicate search effort that the standard split strategy may spend, as shown in Example~\ref{example:cost-split}. 
\begin{property}
Any solution that is compatible with a CT node is compatible with at least one of its child CT nodes.
\label{prop2}
\end{property}
\begin{proof}
The two constraints for resolving a conflict ensure that every solution satisfies at least one of them (because, otherwise, the solution is not conflict-free). Therefore, we only need to prove that any solution $\Pi$ that is compatible with CT node $n$ and satisfies the new constraint imposed on agent $a_i$ is compatible with at least one child CT node in $nodes_i$, where $nodes_i$ is the set of child CT nodes generated after replanning for agent $a_i$. Let $\pi_i$ be the path of agent $a_i$ in solution $\Pi$. Because every CT node in $nodes_i$ differs with CT node $n$ only in the path, constraint, and cost lower bound of agent $a_i$, according to Definition~\ref{def:compatible_old}, we only need to prove that there exists a child CT node in $nodes_i$ that $\pi_i$ is compatible with.

Since $\Piset_i'$ is a cost-unique Pareto-optimal frontier of agent $a_i$, there exists a path $\pi^* \in \Pi_i'$ such that $\mathbf{c}(\pi^*)\preceq \mathbf{c}(\pi_i)$.
With $n.\mathbf{lb}_i \preceq \mathbf{c}(\pi_i)$ and Property~\ref{prop0}, we have $\comax(n.\mathbf{lb}_i, \mathbf{c}(\pi^*)) \preceq \mathbf{c}(\pi_i)$. Since $\comax(n.\mathbf{lb}_i, \mathbf{c}(\pi^*)) \in 
\{\comax(n.\mathbf{lb}_i, \mathbf{c}(\pi)) \mid \pi \in \Piset_i'\}$, there exists a vector in $\mathbf{LB}_i$ that weakly dominates $\comax(n.\mathbf{lb}_i, \mathbf{c}(\pi^*))$, which in turn weakly dominates $\mathbf{c}(\pi_i)$.
Therefore, $\pi_i$ is compatible with the corresponding child CT node.
\end{proof}

\begin{example}
Consider CT node $n_1$ in Figure~\ref{fig:example_ct}.
When replanning for agent $a_1$ with new constraint $\langle a_1, D, 2\rangle$, the cost-unique Pareto-optimal frontier consists of two paths $\pi_1^{(1)} = [A, B, B, D]$ (with $\mathbf{c}(\pi_1^{(1)}) = (4,2)$) and $\pi_1^{(2)}= [A, C, C, D]$ (with $\mathbf{c}(\pi_1^{(2)}) = (2.5,3)$).
For these two paths, we have $ \comax(n_1.\mathbf{lb}_1, \mathbf{c}(\pi_1^{(1)}))=(4,3)$ and  $\comax(n_1.\mathbf{lb}_1, \mathbf{c}(\pi_1^{(2)}))=(2.5,3)$.
Cost splitting creates only one child CT node, whose cost lower bound and path for agent $a_i$ are set to $(2.5,3)$  and $\pi_1^{(2)}$, respectively. 
Conceptually, this child CT node corresponds to CT node $n_4$ in Figure~\ref{fig:example_ct} since these two CT nodes share the same sets of constraints and paths.
Since $(4,3)$ is weakly dominated by $(2.5,3)$, cost splitting does not generate a child CT node for it. 
Intuitively, assuming cost splitting generates this child CT node, which corresponds to CT node $n_3$, every path that is compatible with $n_3$ is also compatible with $n_4$. Therefore, it is not necessary to generate $n_3$.

Consider CT node $n_2$. When replanning for agent $a_1$ with new constraint $\langle a_1, D, 2\rangle$, the cost-unique Pareto-optimal frontier consists of two paths $\pi_1^{(1)}$ and $\pi_1^{(2)}$, too.
For these two paths, we have $\comax(n_2.\mathbf{lb}_1, \mathbf{c}(\pi_1^{(1)}))=(4,2)$ and $\comax(n_2.\mathbf{lb}_1, \mathbf{c}(\pi_1^{(2)}))=(3,3)$, respectively.
Cost splitting generates two child CT node, which corresponds to CT nodes $n_6$ and $n_7$ in Figure~\ref{fig:example_ct}.
The path of $a_1$ in $n_6$ and $n_7$ are set to $\pi_1^{(1)}$ and $\pi_1^{(2)}$, respectively.
Note that the path of $a_1$ (with cost $(2.5, 3)$) in CT node $n_7$ is not compatible with it. 
As we will show later, this does not affect the optimality and completeness of MO-CBS since the path cost dominates $n_7.\mathbf{lb}_1=(3,3)$ and hence dominates all paths of agent $a_1$ that are compatible with  CT node $n_7$.

Overall, MO-CBS with cost splitting generates one fewer CT node (that is, CT node $n_3$) than MO-CBS with standard splitting.
\label{example:cost-split}
\end{example}

\begin{property}
When splitting a CT node, MO-CBS with cost splitting never generates more child CT nodes than MO-CBS with standard splitting.
\end{property}
\begin{proof}
The property holds because, according to Line~\ref{line:LBi} of Algorithm~\ref{alg:cost-split}, the size of $\mathbf{LB}_i$ (that is, the number of CT nodes that MO-CBS with cost splitting generates) is at most the size of $\Piset'_i$  (that is, the number of CT nodes that MO-CBS with standard splitting generates).
\end{proof}

\section{Disjoint Cost Splitting}

In Example~\ref{example:cost-split}, any path of agent $a_1$ whose cost is weakly dominated by $(3,3)$ is compatible with both CT nodes $n_1$ (with $n_1.\mathbf{lb}_1 = (1.5, 3)$) and $n_2$ (with $n_2.\mathbf{lb}_1 = (3, 2)$).
Hence, CT nodes $n_1$ and $n_2$ can still lead to the same solution and thus result in duplicate search effort.
In this section, we describe our second splitting strategy for MO-CBS, disjoint cost splitting, which further reduces duplicate search effort of MO-CBS and ensures that, when splitting a CT node, any solution that is compatible with the CT node is compatible with exactly one of its child CT nodes.

In addition to maintaining cost lower bounds, MO-CBS with disjoint cost splitting also maintains a \emph{cost upper bound} $n.\mathbf{UB}_i$ for each agent $a_i$ in each CT node $n$, that consists of a set of cost vectors that are weakly dominated by $n.\mathbf{lb}_i$. It also modifies Definition~\ref{def:compatible_old} as well as the \textsc{Initialization} and \textsc{Split} functions used by Algorithm~\ref{alg:mo-cbs} as follows.
\begin{definition}
A path $\pi_i$ of agent $a_i$ is compatible with a CT node $n$ iff:
\begin{enumerate}
    \item $n.\mathbf{lb}_i \preceq \mathbf{c}(\pi_i)$;
    \item $\forall \ \mathbf{ub} \in n.\mathbf{UB}_i \, \mathbf{ub} \not \preceq \mathbf{c}(\pi_i)$; and
    \item $\pi_i$ satisfies all constraints of CT node $n$.
\end{enumerate}
A solution is compatible with a CT node iff all its paths are compatible with the CT node.
\label{def:compatible_new}
\end{definition}

In \textsc{Initialization}, MO-CBS with disjoint cost splitting finds a cost-unique Pareto-optimal frontier $\Piset_i$ for each agent $a_i$ and generates a root CT node for each combination of the paths in $\Piset_1\times\Piset_2\ldots\Piset_m$.
Let $[\pi_i^1,\pi_i^2\ldots \pi_i^{|\Piset_i|}]$ denote the paths of $\Piset_i$ sorted in a predetermined order.
For a root CT node $n$ with path $\pi_i^j$ for agent $a_i$, $n.\mathbf{UB}_i$ is initialized to $\ND(\{\comax(\mathbf{c}(\pi_i^j), \mathbf{c}(\pi_i^k))| k=1,2\ldots j - 1\})$.
For example, consider the cost upper bounds of CT nodes $n_1$ and $n_2$ in Figure~\ref{fig:example_ct}. If we sort the paths of $\Piset_1$ in lexicographic order with respect to their path costs (that is, path $\pi_1^1=[A,C,D]$ is before path $\pi_1^2=[A,B,D]$), then
$n_1.\mathbf{UB}_1$ is set to $\emptyset$, and $n_2.\mathbf{UB}_1$ is set to $\ND(\{\comax(\mathbf{c}(\pi_1^2), \mathbf{c}(\pi_1^1))\}) = \{(3,3)\}$. Both $n_1.\mathbf{UB}_2$ and $n_2.\mathbf{UB}_2$ are set to $\emptyset$ because the Pareto-optimal frontier for agent $a_2$ contains only one path.
\begin{property}
Any solution is compatible with exactly one root CT node.
\label{prop2_5}
\end{property}
\begin{proof}
Given a solution $\Pi$, we use  $\pi_i, i=1,2 \ldots m$ to denote the path for agent $a_i$ in solution $\Pi$ and $\pi_i'$ the first path in the sorted sequence of $\Piset_i$ with $\mathbf{c}(\pi_i') \preceq \mathbf{c}(\pi_i)$ during the initialization. We will prove that $\Pi$ is only compatible to the root CT node $n'$ generated with paths $\pi_1',\pi_2'\ldots \pi_m'$. 

We first prove that $\Pi$ is compatible with $n'$.
Since we have $n'.\mathbf{lb}_i = \mathbf{c}(\pi_i') \preceq \mathbf{c}(\pi_i)$, the first condition of Definition~\ref{def:compatible_new} holds. Since $\mathbf{c}(\pi_i') \prec \comax(\mathbf{v}, \mathbf{c}(\pi_i'))$ holds for any vector $\mathbf{v}$ of length $N$ with $\mathbf{c}(\pi_i') \not\preceq \mathbf{v}$, $c(\pi_i') \prec \mathbf{ub}$ holds for any $\mathbf{ub} \in n'.\mathbf{UB}_i$, that is, the second condition holds.The third condition also holds because root CT nodes do not have any constraints. Thus, $\Pi$ is compatible with $n'$.

We then prove that there does not exist another root CT node $n''$ that $\Pi$ is compatible with. If this is not true, then
there must exist some agent $a_i$ for which CT nodes $n'$ and $n''$ have different paths. Let $\pi_i''$ denote the path for agent $a_i$ in CT node $n''$. Since $\pi_i$ is compatible with both root CT nodes $n'$ and $n''$ and the cost lower bounds are equal to the path costs in root CT nodes, $\mathbf{c}(\pi) \preceq \mathbf{c}(\pi_i)$ holds for both $\pi=\pi_i'$ and $\pi=\pi_i''$, and thus, $\mathbf{c}' = \comax(\mathbf{c}(\pi_i''), \mathbf{c}(\pi_i')) \prec \mathbf{c}(\pi_i)$ holds.
Furthermore, since $\pi_i'$ is the first path in the sorted sequence of $\Piset_i$ with $\mathbf{c}(\pi_i') \preceq \mathbf{c}(\pi_i)$, $\pi_i''$ is after $\pi_i'$ in this sorted sequence. As a result, $\mathbf{c}'$ is either in $n''.\mathbf{UB}_i$ or dominated by a vector in $n''.\mathbf{UB}_i$, $\pi_i$ and thus $\pi_i$ is not compatible with $n''$; a contradiction is found.
\end{proof}

\begin{algorithm}[t]
    \footnotesize
   
\DontPrintSemicolon
\SetKwProg{Fn}{Function}{:}{end}
   
  \Fn{\textsc{Split}$(n, \mathit{conf})$}{
        $children \gets \emptyset$\;
         \ForEach{ $a_i$ involved in $\mathit{conf}$ \label{line:disjoint_resolve_conflict}}{
            $\mathit{cons} \gets$ the constraint imposed on $a_i$\;
            $C'_i \gets n.constraints \cup \{\mathit{cons}\}$ \label{line:new_constraint_union}\;
            $\Piset'_i \gets \textsc{LowLevelSearch}(a_i, C'_i)$ \label{line:disjoint_cost_split_lowlevel} \;
            $\mathbf{LB}_i \gets \ND(\{\comax(n.\mathbf{lb}_i, \mathbf{c}(\pi))| \mathbf{\pi} \in \Piset_i'\})$ \;
            $\mathbf{UB} \gets n.\mathbf{UB}_i$ \label{line:init_UB}\;
            \ForEach{ $\mathbf{lb}\in \mathbf{LB}_i$ \label{line:inner_iteration}}{
                $n' \gets n$\;
                $n'.constraints\gets C'_i$\;
                $n'.\mathbf{lb}_i \gets \mathbf{lb}$\;
                $n'.\mathbf{UB}_i \gets \ND(\{\comax(\mathbf{lb}, \mathbf{v}) | \mathbf{v}\in \mathbf{UB}\})$\label{line:nondominated_nodes}\label{line:set_upper_bound}\;
                \lIf{$n'.\mathbf{lb}_i \in n'.\mathbf{UB}_i$ \label{line:delete_useless_node}}{%
                    \textbf{continue}%
                }
             $n'.paths[i]\gets$ a path $ \pi\in \Piset'_i$ such that $ \comax(n.\mathbf{lb}_i, \mathbf{c}(\pi)) = \mathbf{lb}$\;
                add $n'$ to $children$\;
                add $\mathbf{lb}$ to $\mathbf{UB}$ \label{line:add_lower_bound}\;
            }
            
        }
    \Return $children$\;
    }
\caption{Disjoint cost splitting}
\label{alg:disjoint-cost-split}
\end{algorithm}

Algorithm~\ref{alg:disjoint-cost-split} shows the $\textsc{Split}$ function for disjoint cost splitting.
Similar to cost splitting, for each agent $a_i$ involved in the conflict, disjoint cost splitting generates a set of cost lower bounds $\mathbf{LB}_i$ and create a child CT node for each vector in $\mathbf{LB}_i$.
Disjoint cost splitting uses variable $\mathbf{UB}$ to store $n.\mathbf{UB}_i$ and all cost lower bounds of the CT nodes that have been generated.
When generating a CT node $n'$, MO-CBS sets $n'.\mathbf{UB}_i$ to $\ND(\{\comax(n'.\mathbf{lb}_i, \mathbf{v}) | \mathbf{v}\in \mathbf{UB}\})$.
Note that, if $n'.\mathbf{UB}_i$ contains $n'.\mathbf{lb}_i$, there does not exists a path for agent $a_i$ (and hence a solution) that is compatible with CT node $n'$.
In such case, CT node $n'$ is pruned.

\begin{property}
Any solution that is compatible with a CT node is compatible with either exactly one of its child nodes or two of its child nodes that have different constraints.
\label{prop3}
\end{property}
\begin{proof}
Similar to the proof of Property~\ref{prop2}, we only need to prove that any solution $\Pi$ that is compatible with CT node $n$ and satisfies the new constraint imposed on agent $a_i$ is compatible with exactly one child node in $nodes_i$, where $nodes_i$ is the set of child CT nodes generated after replanning agent $a_i$.
We use  $\pi_i$ to denote the path for agent $a_i$ in solution $\Pi$.
From Property~\ref{prop2}, we know that there exists at least one child node in $nodes_i$, denoted as $n'$, such that the first and third conditions of Definition~\ref{def:compatible_new} hold for $\pi_i$ and $n'$.
We let $n'$ be the first such child node generated during the inner iteration of $\textsc{Split}$. 
Consider the iteration when $n'$ is created.
$\mathbf{ub} \not \preceq \mathbf{c}(\pi_i)$ and hence $\comax(n'.\mathbf{lb}_i, \mathbf{ub}) \not \preceq \mathbf{c}(\pi)$ hold for any $\mathbf{ub} \in \mathbf{UB}$. 
Therefore, path $\pi_i$ and also solution $\Pi$ are compatible with CT node $n'$. 

Then, we prove that there exists exactly one such child CT node. Because CT node $n'$ is the first child CT node of $n$ that $\Pi$ is compatible with, we only need to consider those nodes $n''$ generated after $n'$ with $n''.\mathbf{lb}_i \preceq \mathbf{c}(\pi_i)$. Together with $n'.\mathbf{lb}_i \preceq \mathbf{c}(\pi_i)$, we have $\mathbf{c'} = \comax(n'.\mathbf{lb}_i, n''.\mathbf{lb}_i) \preceq \mathbf{c}(\pi_i)$. Since $\mathbf{c'}$ is either in $n''.\mathbf{UB}_i$ or dominated by a vector in $n''.\mathbf{UB}_i$, $\pi_i$ and thus $\Pi$ are not compatible with $n''$. 
\end{proof}

\begin{example}

Consider CT node $n_2$ in Figure~\ref{fig:example_ct}.
When replanning for agent $a_1$ with new constraint $\langle a_1, D, 2\rangle$, 
disjoint cost splitting generates two child CT nodes that correspond to CT nodes $n_6$ and $n_7$ in Figure~\ref{fig:example_ct}.
From Example~\ref{example:cost-split}, we have $n_7.\mathbf{lb}_1 = (3,3)$.
When generating CT node $n_7$, $\mathbf{UB}$ contains $(3,3)$ from $n_2.\mathbf{UB}_1$ and $n_6.\mathbf{lb}_1=(4,2)$.
Therefore, $n_7.\mathbf{UB}_1$ is set to $\ND(\{\comax(n_7.\mathbf{lb}_1, (3,3)), \comax(n_7.\mathbf{lb}_1,$ $ (4,2))\})=\{ (3,3)\}$.
CT node $n_7$ is then pruned because $n_7.\mathbf{lb}_1\in n_7.\mathbf{UB}_1$.
Overall, MO-CBS with disjoint cost splitting generates two fewer CT node (that is, CT nodes $n_3$ and $n_7$) than MO-CBS with standard splitting.
\label{example:disjoint-cost-split}
\end{example}

\section{Theoretical Results}
We now prove the completeness and optimality of both of our splitting strategies. Our proof follows the proof given by original MO-CBS \cite{ren2021multi}. Let $PO$ be a cost-unique Pareto-optimal frontier for a given MO-MAPF instance, $\Solutions$ be the solutions that we have found by far during the search, $\mathbf{S} = \{\mathbf{c}(\Pi) \ | \ \Pi \in \Solutions \}$ be their costs, and $PO|\mathbf{S} = \{\Pi | \Pi \in PO, \mathbf{c}(\Pi) \not \in \mathbf{S}\}$ be the subset of $PO$ whose costs have not been found by far. 


\begin{lemma}
Any CT node that a solution in $PO|\mathbf{S}$ is compatible with is not pruned in Algorithm~\ref{alg:mo-cbs}.
\label{lem1}
\end{lemma}
\begin{proof}
Suppose that a solution $\Pi \in PO|\mathbf{S}$ is compatible with a CT node $n$. For each agent $a_i$ and its corresponding path $\pi_i$ in $\Pi$, we have $n.\mathbf{lb}_i \preceq \mathbf{c}(\pi_i)$ (from Definition~\ref{def:compatible_new}) and $\mathbf{c}(n.paths[i]) \preceq n.\mathbf{lb}_i$ (because $n.\mathbf{lb}_i$ is calculated by taking $\comax$ with $\mathbf{c}(n.paths[i])$ and another vector).
We thus have $\mathbf{c}(n.paths[i]) \preceq \mathbf{c}(\pi_i)$ and hence $\mathbf{c}(n.paths) \preceq \mathbf{c}(\Pi)$. Since $\Pi$ is Pareto-optimal with $\mathbf{c}(\Pi) \notin \mathbf{S}$, there is no solution in $\Solutions$ whose cost weakly dominates $\Pi$ and thus $\mathbf{c}(n.paths)$. Hence, CT node $n$ is not pruned.
\end{proof}

\begin{lemma}
Every solution in $PO|\mathbf{S}$ is compatible with at least one CT node in $Open$ at the end of every iteration of Algorithm~\ref{alg:mo-cbs}.
\label{lem3}
\end{lemma}
\begin{proof}
We prove this lemma by induction.
Properties~\ref{prop1_5} and \ref{prop2_5} show that the lemma holds in \textsc{Initialization} (Algorithm~\ref{alg:mo-cbs}).
We assume that this lemma holds before the $k$-th CT node is extracted from $Open$. When the $k$-th CT node $n$ is extracted, for every solution $\Pi \in PO|\mathbf{S}$, there are three cases:
(1) If $\Pi$ is not compatible with CT node $n$, then, based on the assumption, it is compatible with one CT node in $Open$;
(2) If $\Pi$ is compatible with CT node $n$ and $n.paths$ is conflict-free, then $\mathbf{c}(n.paths)= \mathbf{c}(\Pi)$ because $\mathbf{c}(n.paths) \preceq \mathbf{c}(\Pi)$ and $\Pi$ is Pareto-optimal. So, $n.paths$ is added to $\Solutions$, and $\Pi$ is removed from $PO|\mathbf{S}$; or
(3) If $\Pi$ is compatible with CT node $n$ and $n.paths$ have conflicts, then from Lemma~\ref{lem1} and Properties~\ref{prop2} and \ref{prop3}, $\Pi$ is still compatible with at least one child CT node that is added to $Open$.
Therefore, the lemma holds.
\end{proof}

Our Lemma~\ref{lem3} corresponds to Lemma 3 in \cite{ren2021multi}.
The following theorem combines Theorems 1 and 2 in \cite{ren2021multi} and shows that (disjoint) cost splitting maintains the optimality and completeness of MO-CBS.
The same proofs for Theorems 1 and 2 in \cite{ren2021multi} apply here if combined with our Lemma~\ref{lem3} and hence is omitted. 

\begin{theorem}
MO-CBS with (disjoint) cost splitting finds a cost-unique Pareto-optimal frontier for a given MO-MAPF instance in finite time, if it exists.
\end{theorem}

\begin{figure*}[ht!]
    \centering

    \includegraphics[width=0.9\textwidth]{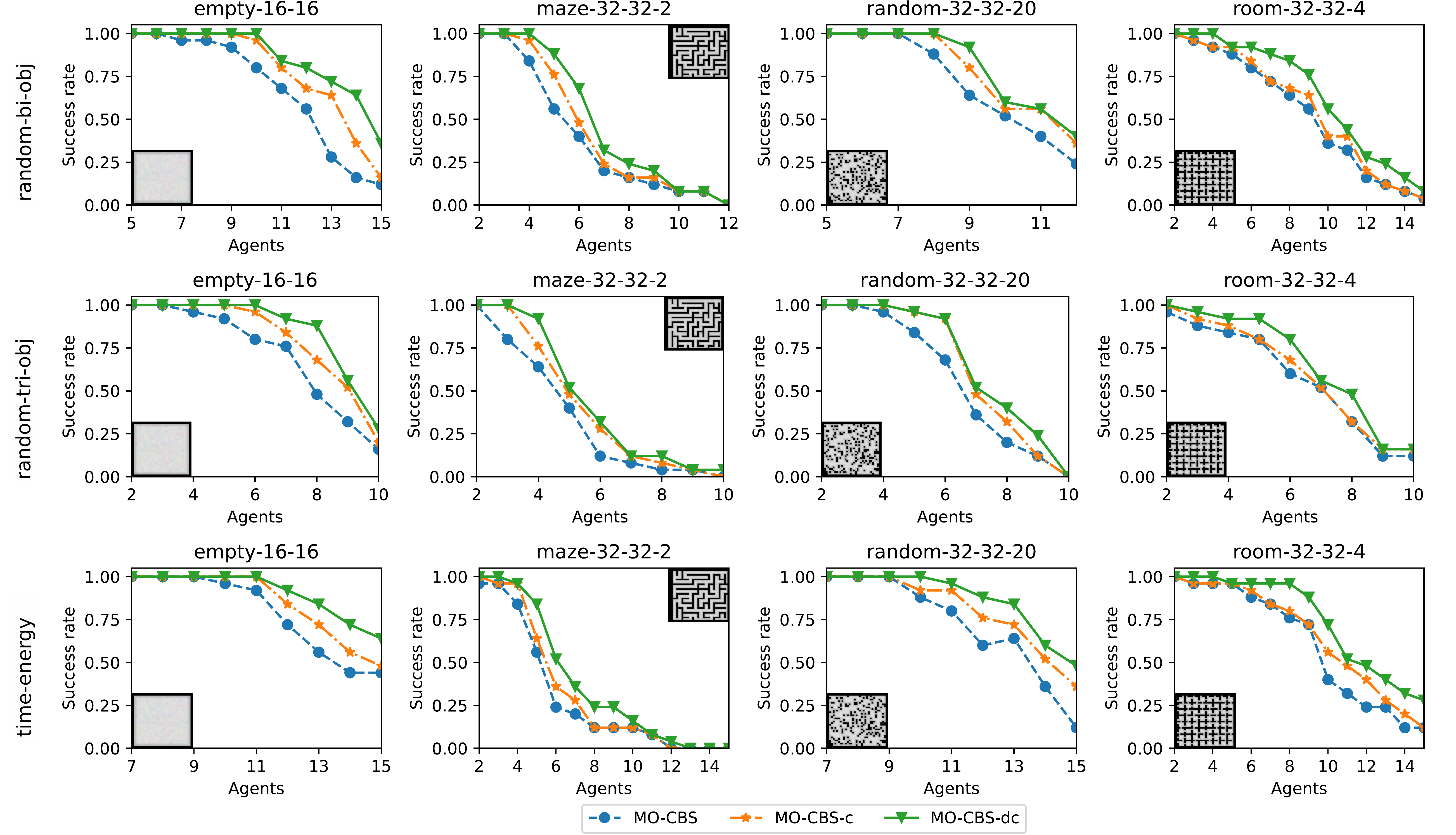}
    \caption{Success rate results for MO-CBS variants on different grid maps with different combinations of objectives. }
    \label{fig:result_all}
\end{figure*}

\section{Experimental Results}

In this section, we experimentally evaluate cost splitting and disjoint cost splitting.
\footnote{Please see Appendix for additional experimental results.}
The variants of MO-CBS are MO-CBS, MO-CBS-c, and MO-CBS-dc, where c adds cost splitting and dc adds disjoint cost splitting.
We implemented all algorithms in Python that share the same code base as much as possible. (Source code and data will be made available to public upon acceptance for publication.)
We ran experiments on a server with a AMD EPYC 7742 CPU. We limited the available memory to 16 GB and the time limit for solving each MO-MAPF instance to 1,000 seconds.

We use three different combinations of cost metrics, which lead to three different combinations of objectives:
\begin{enumerate}
    \item (\textit{random-bi-obj}) Every edge is assigned a 2-dimensional cost vector with each component being an integer randomly sampled from $\{1, 2\}$.
    \item (\textit{random-tri-obj}) Similar to \textit{random-bi-obj} except that every edge is assigned a 3-dimensional cost vector.
    \item (\textit{time-energy}) Each vertex is annotated with an integer indicating its height. We consider two cost metrics, namely \emph{travel time}, that is, the number of timesteps till termination, and total \emph{energy consumption}, where moving upward from a vertex at height $i$ to a vertex at height $j$ costs $j-i$ units of energy, and an action otherwise costs one unit of energy.
    We set the height of a vertex according to vertex's distance to the central point of the map, which creates a hill-like height map. More details are shown in the appendix.
\end{enumerate}

We evaluate the algorithms on four grid maps from the MAPF benchmark~\cite{stern2019mapf},\footnote{\url{https://movingai.com/benchmarks/mapf.html}} namely \emph{empty-16-16}, \emph{maze-32-32-2}, \emph{random-32-32-20}, and \emph{room-32-32-4}.
For each combination of objectives and each grid map, we vary the number of agents and, for each number of agents, average over 25 ``random scenarios'' from the benchmark set.


Figure~\ref{fig:result_all} shows the \emph{success rate}, that is, the percentage of MO-MAPF instances that an algorithm solves within the time limit, for different combinations of objectives, maps, and numbers of agents.
In most cases, MO-CBS-dc has higher success rates than MO-CBS-c, and they both have higher success rates than MO-CBS. 
Figure~\ref{fig:instance_time} shows the individual runtimes (in seconds) of MO-CBS and MO-CBS-dc for all instances.
We use different colors to distinguish instances with different combinations of objectives. 
MO-CBS-dc has a smaller runtime than MO-CBS in almost all instances with a maximum speedup of more than $125$ times.

\begin{figure}[t]
    \centering
    \includegraphics[width=0.26\textwidth]{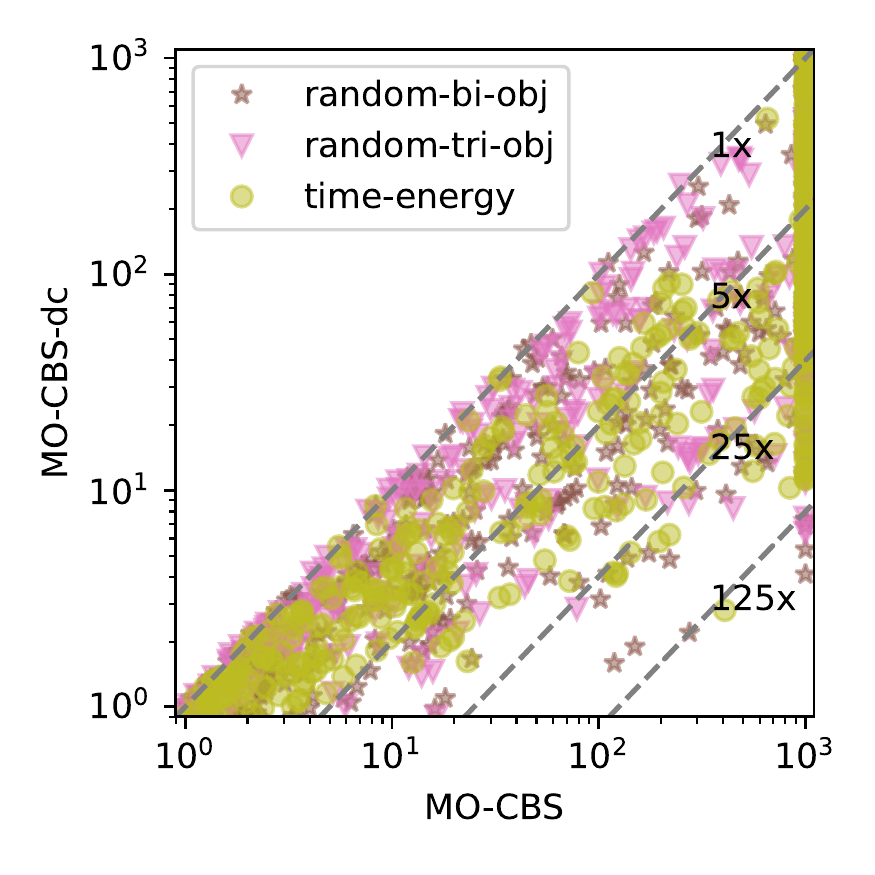}
    \caption{Runtimes of MO-CBS and MO-CBS-dc on all MO-MAPF instances. The $x$- and $y$-coordinates of a dot corresponds to the runtimes (in seconds) of MO-CBS and MO-CBS-dc on an MO-MAPF instance, respectively.}
    \label{fig:instance_time}
\end{figure}

\section{Conclusions}

We propose two splitting strategies for MO-CBS, namely cost splitting and disjoint cost splitting, both of which speed up MO-CBS without losing its optimality or completeness guarantees. Our experimental results show that disjoint cost splitting, our best splitting strategy, speeds up MO-CBS by up to two orders of magnitude and substantially improves its success rates in various settings.
Future work includes improving MO-CBS with improvement techniques of CBS, such as adding heuristics~\cite{felner2018adding} and symmetry breaking~\cite{li2019symmetry}.

\bibliography{aaai22}

\begin{thebibliography}{11}
\providecommand{\natexlab}[1]{#1}

\bibitem[{Felner et~al.(2018)Felner, Li, Boyarski, Ma, Cohen, Kumar, and
  Koenig}]{felner2018adding}
Felner, A.; Li, J.; Boyarski, E.; Ma, H.; Cohen, L.; Kumar, T. K.~S.; and
  Koenig, S. 2018.
\newblock Adding heuristics to conflict-based search for multi-agent path
  finding.
\newblock In \emph{{ International Conference on Automated Planning and
  Scheduling (ICAPS)}}, 83--87.

\bibitem[{Li et~al.(2019)Li, Harabor, Stuckey, Ma, and Koenig}]{li2019symmetry}
Li, J.; Harabor, D.; Stuckey, P.~J.; Ma, H.; and Koenig, S. 2019.
\newblock Symmetry-Breaking Constraints for Grid-Based Multi-Agent Path
  Finding.
\newblock In \emph{{AAAI Conference on Artificial Intelligence (AAAI)}},
  6087--6095.

\bibitem[{Ma et~al.(2016)Ma, Tovey, Sharon, Kumar, and Koenig}]{ma2016multi}
Ma, H.; Tovey, C.; Sharon, G.; Kumar, T. K.~S.; and Koenig, S. 2016.
\newblock Multi-Agent Path Finding with Payload Transfers and the
  Package-Exchange Robot-Routing problem.
\newblock In \emph{{AAAI Conference on Artificial Intelligence (AAAI)}},
  3166--3173.

\bibitem[{Morris et~al.(2016)Morris, Pasareanu, Luckow, Malik, Ma, Kumar, and
  Koenig}]{morris2016planning}
Morris, R.; Pasareanu, C.~S.; Luckow, K.; Malik, W.; Ma, H.; Kumar, T. K.~S.;
  and Koenig, S. 2016.
\newblock Planning, Scheduling and Monitoring for Airport Surface Operations.
\newblock In \emph{AAAI-16 Workshop on Planning for Hybrid Systems}.

\bibitem[{Ren, Rathinam, and Choset(2021{\natexlab{a}})}]{ren2021multi}
Ren, Z.; Rathinam, S.; and Choset, H. 2021{\natexlab{a}}.
\newblock Multi-objective conflict-based search for multi-agent path finding.
\newblock In \emph{{{IEEE} International Conference on Robotics and Automation
  ({ICRA})}}, 8786--8791.

\bibitem[{Ren, Rathinam, and
  Choset(2021{\natexlab{b}})}]{ren2021subdimensional}
Ren, Z.; Rathinam, S.; and Choset, H. 2021{\natexlab{b}}.
\newblock Subdimensional expansion for multi-objective multi-agent path
  finding.
\newblock \emph{{IEEE Robotics and Automation Letters}}, 6(4): 7153--7160.

\bibitem[{Sharon et~al.(2015)Sharon, Stern, Felner, and
  Sturtevant}]{sharon2015conflict}
Sharon, G.; Stern, R.; Felner, A.; and Sturtevant, N.~R. 2015.
\newblock Conflict-based search for optimal multi-agent pathfinding.
\newblock \emph{Artificial Intelligence}, 219: 40--66.

\bibitem[{Stern et~al.(2019)Stern, Sturtevant, Atzmon, Walker, Li, Cohen, Ma,
  Kumar, Felner, and Koenig}]{stern2019mapf}
Stern, R.; Sturtevant, N.~R.; Atzmon, D.; Walker, T.; Li, J.; Cohen, L.; Ma,
  H.; Kumar, T. K.~S.; Felner, A.; and Koenig, S. 2019.
\newblock Multi-Agent Pathfinding: Definitions, Variants, and Benchmarks.
\newblock In \emph{{Symposium on Combinatorial Search ({socs})}}, 151--158.

\bibitem[{Wagner and Choset(2015)}]{wagner2015subdimensional}
Wagner, G.; and Choset, H. 2015.
\newblock Subdimensional expansion for multirobot path planning.
\newblock \emph{{Artificial intelligence}}, 219: 1--24.

\bibitem[{Wurman, D'Andrea, and Mountz(2008)}]{wurman2008coordinating}
Wurman, P.~R.; D'Andrea, R.; and Mountz, M. 2008.
\newblock Coordinating Hundreds of Cooperative, Autonomous Vehicles in
  Warehouses.
\newblock \emph{AI Magazine}, 29(1): 9--20.

\bibitem[{Yu and LaValle(2013)}]{yu2013structure}
Yu, J.; and LaValle, S.~M. 2013.
\newblock Structure and Intractability of Optimal Multi-Robot Path Planning on
  Graphs.
\newblock In \emph{{AAAI Conference on Artificial Intelligence (AAAI)}},
  1443--1449.

\end{thebibliography}

\end{document}


\maketitle

In the appendix, we provide implementation details and additional experimental results that are not included in the main document.

\section{Implementation Details}

\begin{figure}[t]
    \centering
    \includegraphics[width=0.26\textwidth]{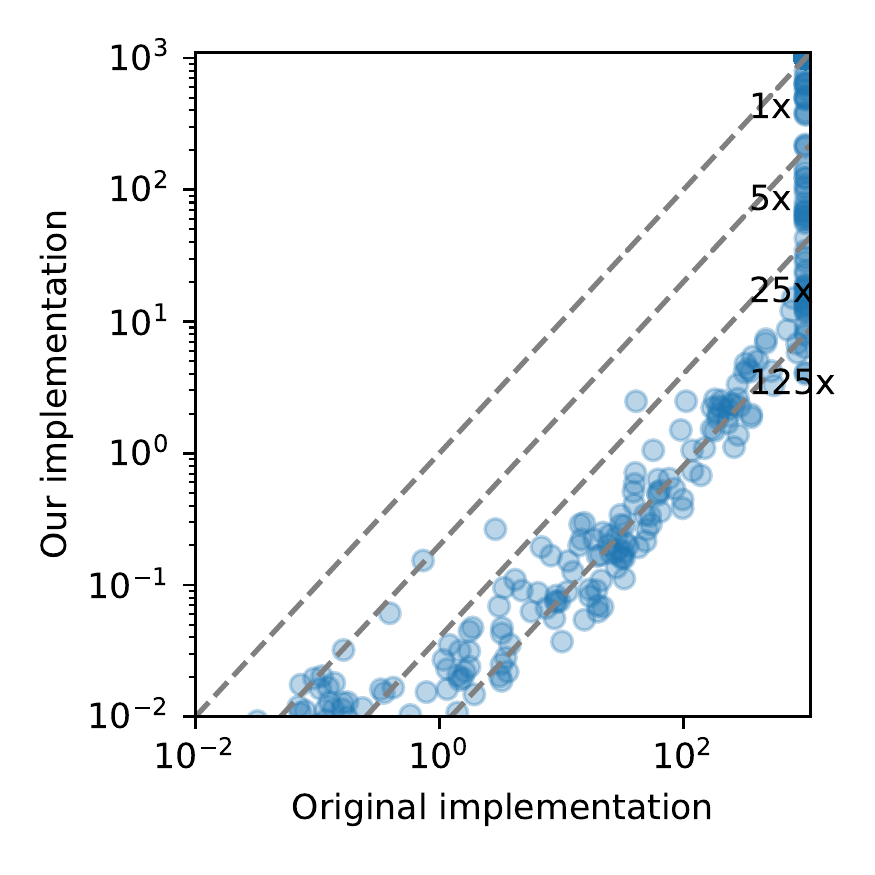}
    \caption{Runtimes (in seconds) of the original MO-CBS implementation and our MO-CBS implementation on individual MO-MAPF instances. 
    }
    \label{fig:compare}
\end{figure}

Our Python implementations of MO-CBS, MO-CBS-c, and MO-CBS-dc are partly based on the Python implementation made available online by the authors of
MO-CBS~\cite{ren2021multi}.\footnote{\url{https://github.com/wonderren/public_pymomapf}} However, we have made the following changes to improve the efficiency of the original implementation:
\begin{enumerate}
    \item While the original implementation uses zero heuristic for the low level of MO-CBS,
    we implemented the same heuristic as~\cite{sharon2015conflict}, which calculates the minimum cost from any state to the goal state for each objective and each agent that ignores the constraints.
    Such heuristic can be efficiently computed by by running the  Dijkstra's algorithm from the goal vertex of an agent.
    It is also a common practice to implement such heuristics for CBS and its variants
    since the heuristic can speed-up the low level substantially and only needs to be computed once.
    \item We found that MO-CBS often calls the low level search for the same agent and set of constraints multiple times. 
    This is partly because of the multiple root CT nodes that MO-CBS has.
    To speed up MO-CBS, we cache the Pareto-optimal frontiers computed by the low level in a hash table, using the indices of the agents and the sets of all constraints imposed on them as keys.
    
\end{enumerate}

Figure~\ref{fig:compare} shows the individual runtimes (in seconds) of the original MO-CBS implementation and our new MO-CBS implementation  (without any improvement that we proposed) for \textit{random-bi-obj} instances on grid map \textit{empty-16-16}. 
Our new implementation runs significantly faster than the original implementation and is hence used in our experimental evaluation.

\section{Additional Experimental Results}

\begin{figure}[t]
    \centering
    \includegraphics[width=0.4\textwidth]{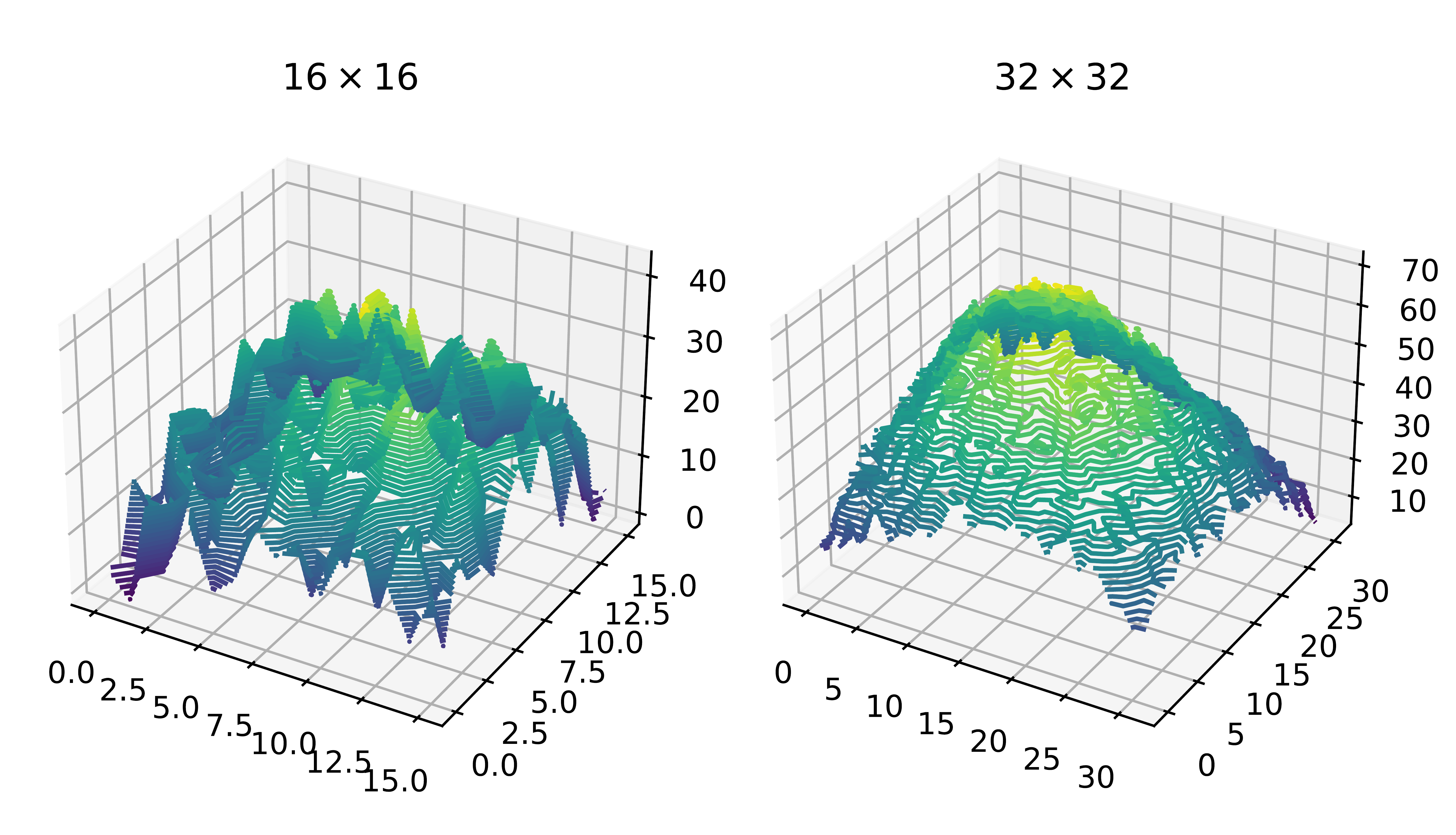}
    \caption{Height maps that are used in the experimental evaluation.}
    \label{fig:hill}
\end{figure}

For the third combination of objectives \textit{time-energy}, we created two height map, a $16\times 16$ one for grid map \emph{empty-16-16} and a $32\times 32$ one for grid maps \emph{maze-32-32-2}, \emph{random-32-32-20}, and \emph{room-32-32-4}.
Figure~\ref{fig:hill} shows the visualization of the two height maps we created.

\begin{table}[t]
    \centering
    \begin{tabular}{|l|rrr|}
    \hline
         & MO-CBS & MO-CBS-c & MO-CBS-dc \\
         \hline
         \hline
        random-bi-obj & 5.28 & 3.42 & \textbf{2.72} \\
        random-tri-obj & 16.31 & 6.34 & \textbf{3.60}  \\
        time-energy &  4.06 & 2.92 & \textbf{2.52} \\
        \hline
        
    \end{tabular}
    \caption{Average branching factors for each algorithm and each combination of objectives.}
    \label{tab:branch}
\end{table}

\begin{figure*}[ht!]
    \centering

    \includegraphics[width=0.9\textwidth]{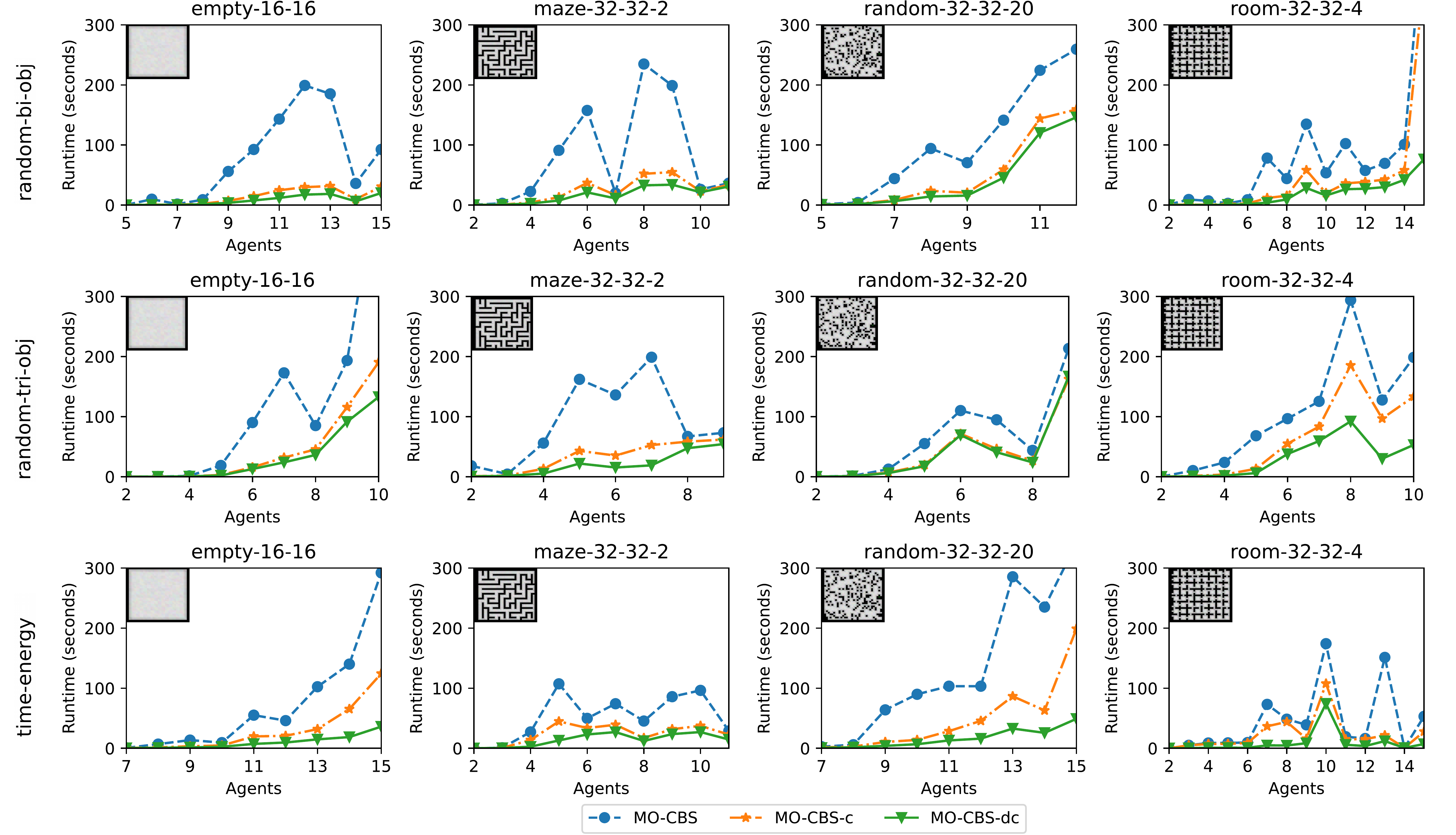}
    \caption{Average runtime over instances that are solved by all MO-CBS variants on different grid maps with different combinations of objectives. }
    \label{fig:result_all_time}
\end{figure*}

\begin{figure*}[ht!]
    \centering

    \includegraphics[width=0.9\textwidth]{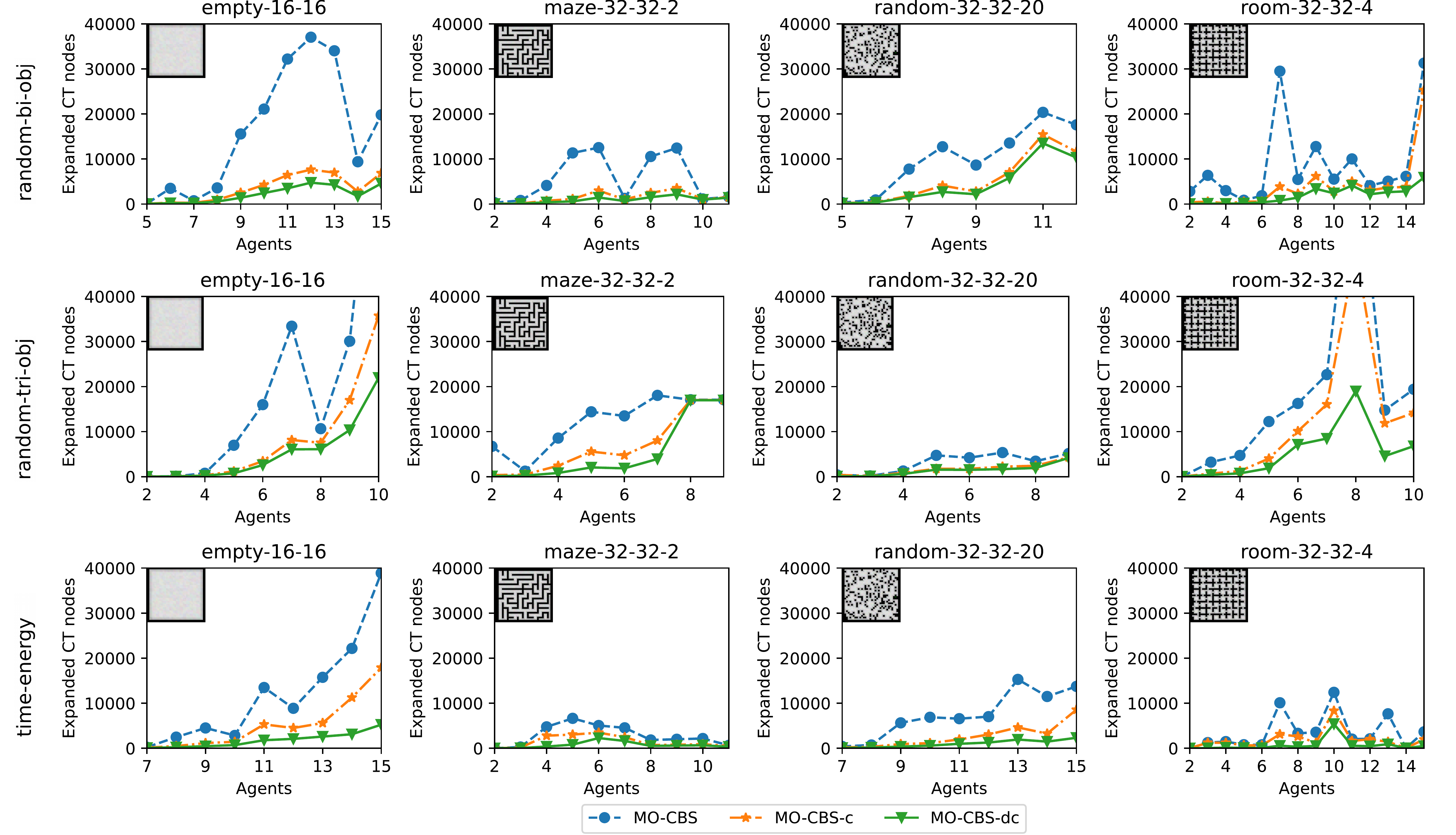}
    \caption{Average node expansion over instances that are solved by all MO-CBS variants on different grid maps with different combinations of objectives. }
    \label{fig:result_all_exp}
\end{figure*}

Table~\ref{tab:branch} shows the average branching factor, where branching factor is the number of child CT nodes that are returned by the  \textsc{Split} function when expanding a CT node, for different combinations of objectives.
In all cases, MO-CBS-dc has smaller branching factors than MO-CBS-c, and they both have smaller branching factors than MO-CBS. 
Comparing \textit{random-bi-obj} and \textit{random-tri-obj}, the reduction of branching factor by cost splitting or disjoint cost splitting becomes much more significant as the number of objectives increases.
This is because, with more objectives, the sizes of cost-unique Pareto-optimal frontiers for each agent is also larger, and cost splitting or disjoint cost splitting becomes more important for reducing numbers of generated child CT nodes.

Figures~\ref{fig:result_all_time} and~\ref{fig:result_all_exp} show the average runtimes and average CT node expansions, respectively, for different algorithms, combinations of objectives, maps, and numbers of agents.
The runtimes and CT node expansions are averaged over all instances solved by all three algorithms. 
Similar to the trend we saw in the branching factor results, in almost all instances, MO-CBS-dc has smaller average runtimes (resp. fewer average CT node expansions) than MO-CBS-c, and they both have smaller average runtimes (resp. fewer average CT node expansions) than MO-CBS.

\bibliography{aaai22}